\documentclass[11pt]{article}
\usepackage{fullpage}

\usepackage{booktabs}
\usepackage[utf8]{inputenc}
\usepackage[T1]{fontenc}
\usepackage{hyperref}
\usepackage{url}     
\usepackage{booktabs}
\usepackage{amsfonts}

\usepackage{amsmath,amsthm}
\usepackage{color}
\usepackage{graphicx}
\usepackage{cite}

\newcommand{\A}{{\cal A}}
\newcommand{\C}{{\cal C}}
\newcommand{\D}{{\cal D}}
\renewcommand{\H}{{\cal H}}
\renewcommand{\S}{{\cal S}}
\newcommand{\X}{{\cal X}}
\newcommand{\er}{\textup{er}}

\newcommand{\card}[1] {\left\vert #1 \right\vert}
\newcommand{\set}[1] {\left\{ #1 \right\}}

\newcommand{\norm}[1] {\left\| #1 \right\|}
\newcommand{\grad}{\nabla}

\newcommand{\N}{\mathbb{N}}

\newcommand{\eps}{\varepsilon}

\DeclareMathOperator*{\E}{\textnormal{\bf E}}
\let\Pr\relax
\DeclareMathOperator*{\Pr}{\textnormal{\bf Pr}}

\newtheorem*{definition}{Definition}
\newtheorem{theorem}{Theorem}
\newtheorem*{theorem*}{Theorem}
\newtheorem{lemma}[theorem]{Lemma}
\newtheorem*{lemma*}{Lemma}
\newtheorem{proposition}[theorem]{Proposition}
\newtheorem*{proposition*}{Proposition}

\newcounter{algorithm}

\newcommand{\alg}{\textsc{Multiaccuracy Boost}}

\newcommand*\samethanks[1][\value{footnote}]{\footnotemark[#1]}

\begin{document}
\title{Multiaccuracy: Black-Box Post-Processing for Fairness in Classification}

\author{Michael P.~Kim\thanks{These authors contributed equally.}~\thanks{Supported by NSF Grant CCF-1763299.}\\
\texttt{mpk@cs.stanford.edu}
\and 
Amirata Ghorbani\samethanks[1]\\
\texttt{amiratag@stanford.edu}
\and 
James Zou\\
\texttt{jamesz@stanford.edu}
}

\date{}

\maketitle

\begin{abstract}
Prediction systems are successfully deployed in applications ranging from disease diagnosis, to predicting credit worthiness, to image recognition.  Even when the overall accuracy is high, these systems may exhibit systematic biases that harm specific subpopulations; such biases may arise inadvertently due to underrepresentation in the data used to train a machine-learning model, or as the result of intentional malicious discrimination.
We develop a rigorous framework of \emph{multiaccuracy} auditing and post-processing to ensure accurate predictions across \emph{identifiable subgroups}.

Our algorithm, \alg, works in any setting where we have black-box access to a predictor and a relatively small set of labeled data for auditing;
importantly, this black-box framework allows for improved fairness and accountability of predictions, even when the predictor is minimally transparent.
We prove that \alg~converges efficiently and show that if the initial model is accurate on an identifiable subgroup, then the post-processed model will be also.
We experimentally demonstrate the effectiveness of the approach to improve the accuracy among minority subgroups in diverse applications (image classification, finance, population health).
Interestingly, \alg~can improve subpopulation accuracy (e.g. for ``black women'') even when the sensitive features (e.g. ``race'', ``gender'') are not given to the algorithm explicitly.
\end{abstract}

\section{Introduction}

Despite the successes of machine learning at
complex tasks that involve making predictions about people,
there is growing evidence that
``state-of-the-art'' models can perform
significantly less
accurately on minority populations than on the majority
population.  Indeed, a notable study of three
commercial face recognition systems known as the
``Gender Shades'' project \cite{ppb}, demonstrated
significant performance gaps across different
populations at classification tasks.  While
all systems achieved roughly 90\% accuracy at
gender detection on a popular benchmark, a closer
investigation revealed that the system was
significantly less accurate on female subjects compared to males
and on dark-skinned individuals compared to light-skinned.
Worse yet, this discrepancy
in accuracy compounded when comparing dark-skinned females
to light-skinned males; classification accuracy
differed between these groups by as much as $34\%$!
This study confirmed empirically the intuition that
machine-learned classifiers may optimize
predictions to perform well on the majority population,
inadvertently hurting performance on the minority
population in significant ways.

A first approach to address this serious problem
would be to update the training distribution to
reflect the distribution of people, making sure
historically-underrepresented populations are
well-represented in the training data.
While this
approach may be viewed as an eventual goal, often
for historical and social reasons, data from
certain minority populations is less available than from
the majority population.  In particular, we may
not immediately have enough data from these
underrepresented subpopulations
to train a complex model.
Additionally, even when adequate representative data is available, this process necessitates
retraining the underlying prediction model.
In the common setting where
the learned model is provided as a service, like
a commercial image recognition system, there may
not be sufficient incentive
(financial, social, etc.) for the service provider
to retrain the model.  Still, the clients of the
model may want to improve the accuracy of the
resulting predictions across the population, even
when they are not privy to the inner workings of the prediction system.

At a high level, our work focuses on a setting, adapted from \cite{multi}, that is common in practice but distinct from much of the other literature on fairness in classification.
We are given black-box access to a classifier, $f_0$, and a relatively small ``validation set" of labeled samples
drawn from some representative distribution $\D$;
our goal is to \emph{audit} $f_0$ to determine whether the predictor satisfies
a strong notion of subgroup fairness, \emph{multiaccuracy}.
Multiaccuracy requires (in a sense that we make formal in Section~\ref{sec:def}) that predictions be unbiased, not just overall, but on every identifiable subpopulation.
If auditing reveals that the predictor does not satisfy multiaccuracy, we aim to
\emph{post-process} $f_0$ to produce a new classifier $f$ that is multiaccurate,
without adversely affecting the subpopulations where $f_0$ was already accurate.

Even if the initial classifier $f_0$ was trained in good faith, it may still exhibit biases on significant subpopulations when evaluated on samples from $\D$.  This setting can arise when minority populations are underrepresented in the distribution used to train $f_0$ compared to the desired distribution $\D$, as in the Gender Shades study \cite{ppb}.
In general, we make no assumptions about how $f_0$ was trained. In particular, $f_0$ may be an adversarially-chosen classifier, which explicitly aims to give erroneous predictions within some protected subpopulation while satisfying marginal statistical notions of fairness.
Indeed, the influential work on ``Fairness Through Awareness'' \cite{fta}, followed by \cite{kearns2017preventing,multi}, demonstrated the weakness of statistical notions of fairness (such as statistical parity, equalized odds, and calibration),
showing that prediction systems can exhibit material forms of discrimination against protected populations, even though they satisfy statistical fairness conditions.  Left unaddressed, such forms of discrimination may discourage the participation of minority populations, leading to further underrepresentation of these groups.
Thus, our goal will be to mitigate systematic biases broadly enough to handle inadvertent and malicious forms of discrimination.

\paragraph{\textbf{Our contributions}}
We investigate a notion of fairness -- multiaccuracy -- originally proposed in \cite{multi}, and develop a framework for auditing and post-processing for multiaccuracy.
We develop a new algorithm, \alg, where a simple learning algorithm -- the auditor -- is used to identify subpopulations in $\D$ where $f_0$ is systematically making more mistakes. This information is then used to iteratively post-process $f_0$ until the multiaccuracy condition -- unbiased predictions in each identifiable subgroup -- is satisfied. Our notion of multiaccuracy differs from parity-based notions of fairness, and is reasonable in settings such as gender detection where we would like to boost the classifier's accuracy across many subgroups. We prove convergence guarantees for \alg~ and show that post-processing for multiaccuracy may actually improve the \emph{overall} classification accuracy.  We describe the post-processing algorithm in Section~\ref{sec:alg}.

Empirically, we validate \alg~ in
several different case studies: gender detection from images as in Gender Shades \cite{ppb}, a semi-synthetic medical diagnosis task, and
adult income prediction.
In all three cases, we use standard, initial prediction models
that achieve good overall classification error but
exhibit biases against significant subpopulations.
After post-processing, the accuracy improves
across these minority groups, even though minority-status
is not explicitly given to the post-processing algorithm
as a feature.  
As long as there are features in the audit set correlated with the (unobserved) human categories, then \alg~ is effective at improving the classification accuracy across these categories.
As suggested by the theory,
\alg~actually improves the overall accuracy, by identifying subpopulations where the initial models systematically erred;
further, post-processing does not
significantly affect performance on groups where accuracy was
already high.
We show that \alg, which only accesses $f_0$ as a black-box,
performs comparably and sometimes even better than very strong white-box alternatives which has full access to $f_0$.
These results are reported in Section~\ref{sec:experiments}.

In Section~\ref{sec:experiments:gender}, we explore the gender detection example further, investigating some of the practical aspects of multiaccuracy auditing and post-processing.
In particular, we observe that the representation of images used for auditing (and post-processing) matters;
we show that auditing is more effective when using an embedding of the images that was trained using an unsupervised autoencoder compared to using the internal representation of the neural network used for prediction. These findings seem consistent with the guiding philosophy, put forth by \cite{fta}, that maintaining ``awareness'' is paramount to detecting unfairness.
We also show that the auditing process, which we use algorithmically as a way to boost the accuracy of the classifier, can also be used to help people understand why their prediction models are making mistakes.
Specifically, the output of the multiaccuracy auditor can be used to produce examples of inputs where the predictor is erring significantly; this provides human interpretation for biases of the original classifier.

\section{Setting and multiaccuracy}
\label{sec:def}
\paragraph{High-level setting.}
Let $\X$ denote the input space; we denote by
$y: \X \to \set{0,1}$ the function that maps
inputs to their label.  Let $\D$ represent
the validation data distribution supported on
$\X$; the distribution $\D$ can be viewed as
the ``true" distribution, on which we will
evaluate the accuracy of the final model.
In particular, we assume that the important
subpopulations are sufficiently represented
on $\D$ (cf. Remark on data distribution).
Our post-processing learner receives as input
a small sample of labeled validation data
$\set{(x,y(x))}$, where $x \sim \D$, as well
as black-box access to
an initial regression / classification model
$f_0:\X\to[0,1]$.  The goal is to output a
new model (using calls to $f_0$) that
satisfies the multiaccuracy fairness
conditions (described below).

Importantly, we make no further assumptions
about $f_0$.  Typically, we will think
of $f_0$ as the output of a learning
algorithm, trained on some other
distribution $\D_0$
(also supported on $\X$); in this scenario, our goal is
to mitigate any inadvertently-learned
biases.  That said, another important
setting assumes that $f_0$ is chosen
\emph{adversarially} to discriminate against
a protected population of individuals,
while aiming to appear accurate and fair
on the whole; here, we aim to protect
subpopulations against malicious misclassification.
The formal guarantees of multiaccuracy provide meaningful protections from both
of these important forms of discrimination.

\paragraph{Additional Notation.}
For a subset $S \subseteq \X$,
we use $x \sim S$ to denote a sample from $\D$ conditioned
on membership in $S$.  We take the characteristic function
of $S$ to be $\chi_S(x) = 1$ if $x \in S$ and $0$ otherwise.
For a hypothesis $f:\X\to[0,1]$, we
denote the classification error of $f$ with respect
to a subset $S \subseteq \X$ as
$\er_S(f;y) = \Pr_{x \sim S}[\bar{f}(x) \neq y(x)]$,
where $\bar{f}(x)$ rounds
$f(x)$ to $\set{0,1}$.
For a function $z:\X\to [-1,1]$ and a subset $S \subseteq \X$,
let $z_S$ be the restriction to $S$ where $z_S(x) = z(x)$
if $x \in S$ and $z_S(x) = 0$ otherwise.
We use $\ell_\D(f;y) = \E_{x \sim \D}[\ell_x(f;y)]$ to denote
the expected cross-entropy loss of $f$ on $x \in \X$
where $\ell_x(f;y) = -y(x) \cdot \log(f(x))
- (1-y(x))\cdot \log(1-f(x))$.

\subsection{Multiaccuracy}
The goal of multiaccuracy is to achieve
low classification error, not just on $\X$ overall,
but also on subpopulations of $\X$.
This goal is formalized in the following definition adapted from
\cite{multi}. 
\begin{definition}[Multiaccuracy]
Let $\alpha \ge 0$ and let $\C \subseteq [-1,1]^{\X}$
be a class of functions on $\X$.
A hypothesis $f:\X \to [0,1]$ is \emph{$(\C,\alpha)$-multiaccurate}
if for all $c \in \C$:
\begin{equation}
\E_{x \sim \D}[c(x) \cdot (f(x) - y(x))]
\le \alpha. 
\end{equation}
\end{definition}
$(\C,\alpha)$-multiaccuracy guarantees that a hypothesis
appears unbiased according to a class of statistical tests
defined by $\C$.  As an example, we could define the class
in terms of a collection of subsets $S \subseteq \X$,
taking $\C$ to be $\chi_S$ (and its negation)
for each subset in the collection; in this case,
$(\C,\alpha)$-multiaccuracy guarantees that for each
$S$, the predictions of $f$ are at most $\alpha$-biased.

Ideally, we would hope to take $\C$ to be the
class of \emph{all} statistical tests. Requiring
multiaccuracy with respect to such a $\C$, however,
requires learning the function $y(x)$ exactly,
which is information-theoretically impossible
from a small sample.  
In practice, if we take $\C$ to be a
\emph{learnable} class of functions,
then $(\C,\alpha)$-multiaccuracy
guarantees accuracy on all \emph{efficiently-identifiable}
subpopulations.

For instance, if we took $\C$ to be the class of width-$4$ conjunctions, then multiaccuracy guarantees unbiasedness, not just on the marginal populations defined by race and separately gender, but by the subpopulations defined by the intersection of race, gender, and two other (possibly ``unprotected") features.
In particular, the subpopulations that multiaccuracy protects can be overlapping and include groups beyond traditionally-protected populations.
This form of computationally-bounded intersectionality provides strong protections against forms of discrimination, like subset targeting, discussed in \cite{fta,multi}.

\subsection{Classification accuracy from multiaccuracy}
Multiaccuracy guarantees that the predictions of a classifier appear unbiased on a rich class of subpopulations; ideally though, we would state a guarantee in terms of the classification accuracy, not just the bias.
Intuitively, as we take $\C$ to define a richer class of tests,
the guarantees of multiaccuracy become stronger.
This intuition is formalized in the following proposition.
\begin{proposition}
\label{prop:ac2er}
Let $\hat{y}:\X \to \set{-1,1}$ as $\hat{y}(x)=1-2y(x)$.
Suppose that for $S \subseteq \X$ with $\Pr_{x \sim \D}[x \in S] \ge \gamma$,
there is some $c \in \C$
such that $\E_{x\sim \D}[\card{c(x) - \hat{y}_S(x)}] \le \tau$.
Then if $f$ is $(\C,\alpha)$-multiaccurate,
$\er_S(f;y) \le 2\cdot(\alpha + \tau)/\gamma$.
\end{proposition}
That is, if there is a function in $\C$
that correlates well with the label function
on a significant subpopulation $S$,
then multi-accuracy
translates into a guarantee on the \emph{classification
accuracy} on this subpopulation. 

\paragraph{Remark on data distribution.}
Note that in our definition of multiaccuracy, we take an
expectation over the distribution $\D$ of validation data.
Ideally, $\D$ should reflect the true population
distribution or could be aspirational, increasing
the representation of populations
who have experienced historical discrimination; for instance,
the classification error guarantee of Proposition~\ref{prop:ac2er}
improves as $\gamma$, the density of the protected subpopulation $S$, grows.
Still, if we take the multiaccuracy error term $\alpha$ small enough, then
we may still hope to improve the accuracy on
less-represented subpopulations.
Our experiments suggest that applying the multiaccuracy framework with an unbalanced valdiation distribution may still help improve the accuracy on underrepresented groups.

\subsection{Auditing for multiaccuracy}
With the definition of $(\C,\alpha)$-multiaccuracy in place, a natural question to ask is how to test if a hypothesis $f$ satisfies the definition; further, if $f$ does not satisfy $(\C,\alpha)$-multiaccuracy, can we update $f$ efficiently to satisfy the definition, while maintaining the overall accuracy?
We will use a learning algorithm $\A$ to audit a classifier $f$ for
multiaccuracy.  The algorithm $\A$ receives a small
sample from $\D$ and aims to learn a function $h$
that correlates with the \emph{residual} function $f-y$.
In Section~\ref{sec:alg}, we describe how to use such an auditor to solve the post-processing problem.
This connection between subpopulation fairness and learning is also made in \cite{kearns2017preventing,multi,ftba}, albeit for different tasks.

\begin{definition}[Multiaccuracy auditing]
Let $\alpha > 0, m \in \N$, and let $\A:(\X \times [-1,1])^m \to [-1,1]^\X$ be a
learning algorithm.  Suppose $D \sim \D^m$ is a set of
independent samples.
A hypothesis $f:\X\to[0,1]$ passes
\emph{$(\A,\alpha)$-multiaccuracy auditing} if for
$h = \A(D; f-y)$:
\begin{equation}
\E_{x\sim\D}\left[h(x)\cdot(f(x)-y(x))\right]
\le \alpha.
\end{equation}
\end{definition}
A special case of $(\A,\alpha)$-multiaccuracy
auditing uses a naive learning algorithm that
iterates over statistical tests $c \in \C$.
Concretely, in our experiments,
we audit with
ridge regression and decision tree regression;
both auditors are effective at identifying
subpopulations on which the model is
underperforming.  Further, in the image recognition setting, we show that auditing can be used to produce interpretable synopses of the types of mistakes a predictor makes.

\section{Post-processing for multiaccuracy}
\label{sec:alg}
Here, we describe an algorithm, \alg, for post-processing a
pre-trained model to achieve multiaccuracy.  The
algorithm is given black-box access to an initial hypothesis
$f_0:\X \to [0,1]$
and a learning algorithm
$\A:(\X\times [-1,1])^m \to [-1,1]^\X$, and for any accuracy parameter
$\alpha > 0$, outputs a hypothesis $f:\X \to [0,1]$
that passes $(\A,\alpha)$-multiaccuracy auditing.
The post-processing algorithm is an iterative procedure
similar to boosting \cite{freund1995desicion,schapire2012boosting}, that uses the multiplicative weights
framework to improve suboptimal predictions identified
by the auditor.  This approach is similar to the algorithm
given in \cite{multi} in the context of fairness
and \cite{trevisan2009regularity} in the context of pseudorandomness.
Importantly, we adapt these algorithms so that \alg~exhibits
what we call the ``do-no-harm'' guarantee; informally,
if $f_0$ has low classification error on some subpopulation
$S \subseteq \X$ identified by $\A$, then the resulting
classification error on $S$ cannot increase significantly.
In this sense, achieving our notion
of fairness need not adversely affect the utility of
the classifier.

A key algorithmic challenge is to learn a multiaccurate predictor without overfitting to the small sample of validation data.
In theory, we prove bounds on the sample complexity necessary to guarantee good generalization as a function of the class $\C$, the error parameter $\alpha$, and the size of subpopulations we wish to protect $\gamma$.
In practice, we need to balance the choice of $\C$ (or $\A$) and the number of iterations of our algorithm to make sure that the auditor is discovering true signal, rather than noise in the validation data.
Indeed, if the auditor $\A$ learns an expressive enough class of functions, then our algorithm will start to overfit at some point; 
we show empirically that multiaccuracy post-processing improves the generalization error before overfitting.
Next, we give an overview of the algorithm, and state its formal guarantees in Section~\ref{sec:theory}.

\begin{figure}[t!]
{\refstepcounter{algorithm} \label{alg}{\bf Algorithm~\thealgorithm:}} \alg 

\fbox{\parbox{\columnwidth}{
{\bf Given:}
\begin{itemize}
\item initial hypothesis $f_0:\X\to[0,1]$
\item auditing algorithm $\A:(\X\times[-1,1])^m \to [-1,1]^\X$
\item accuracy parameter $\alpha > 0$
\item validation data
$D = D_0,\hdots,D_T \sim \D^m$
\end{itemize}

{\bf Let:}
\begin{itemize}
\item $\X_0 \gets \set{x \in \X : f_0(x) \le 1/2}$
\item $\X_1 \gets \set{x \in \X : f_0(x) >1/2}$ \hfill\texttt{// partition X according to f0}
\item $\S \gets \set{\X,\X_0,\X_1}$
\end{itemize}

{\bf Repeat:} from $t = 0,1,\hdots,T$
\begin{itemize}
\item For $S \in \S$:\hfill\texttt{// audit ft on X,X0,X1 with fresh data}
\\\phantom{For }$h_{t,S} \gets \A(D_t;(f_t-y)_S)$
\item $S^* \gets {\rm argmax}_{S \in \S}\E_{x \sim D_t}[h_{t,S}(x) \cdot(f_t(x)-y(x)) ]$
\hfill\texttt{// take largest residual}
\item if $\E_{x \sim D_t}[h_{t,S^*}(x) \cdot (f_t(x) - y(x))]\le \alpha$:
\\\phantom{For }{\bf return} $f_t$\hfill\texttt{// terminate when at most alpha}
\item $f_{t+1}(x) \propto e^{-\eta h_{t,S^*}(x)}\cdot f_{t}(x)\qquad\qquad\qquad \forall x \in S^*$
\hfill\texttt{// multiplicative weights update}
\end{itemize}
}}
\end{figure}

At a high level, \alg~ starts by partitioning
the input space $\X$ based on the initial classifier $f_0$ into
$\X_0 = \set{x \in \X : f_0(x) \le 1/2}$ and $\X_1 = \set{x \in \X : f_0(x) > 1/2}$; note that we can partition $\X$ simply by calling $f_0$.
Partitioning the search space $\X$
based on the predictions of $f_0$
helps to ensure that the $f$ we output maintains the initial accuracy of $f_0$; in particular, it
allows us to search over just the positive-labeled
examples (negative, resp.) for a way to improve the
classifier.  The initial hypothesis may make false positive predictions
and false negative predictions for very different reasons,
even if in both cases the reason is simple enough to be identified by the auditor.

After partitioning the input space, the procedure iteratively
uses the learning algorithm $\A$ to search over $\X$ (and within the partitions $\X_0,\X_1$)
to find any function which correlates significantly
with the current residual in prediction $f-y$.
If $\A$ successfully returns some function $h:\X \to [-1,1]$
that identifies a significant subpopulation where
the current hypothesis is inaccurate, the algorithm
updates the predictions
multiplicatively according to $h$.
In order to update the predictions simultaneously for all $x \in \X$,
at the $t$th iteration, we build
$f_{t+1}$ by incorporating $h_t$ into
the previous model $f_t$.
This approach of augmenting the model at each iteration is similar to boosting.
To guarantee
good generalization of $h$, we assume that $\A$
uses a fresh sample $D_t \sim \D^m$ per iteration. In practice, when we have few samples, we can put all of our samples in one batch and use noise-addition techniques to reduce overfitting~\cite{dwork2015reusable,russo2015much};
this connection to adaptive data analysis was studied formally
in \cite{multi}.

From the stopping condition, it is clear that when the
algorithm terminates, $f_T$
passes $(\A,\alpha)$-multiaccuracy auditing.
Thus, it remains to bound the number of iterations $T$ before
\alg~ terminates.  Additionally, as described, the algorithm
evaluates statistics like $\E_{x\sim\D}[h(x)\cdot(f(x)-y(x))]$,
which depends on $y(x)$ for all $x \in \X$; we need to bound the
number of validation samples we need to guarantee good generalization to the unseen population.
Theorem~\ref{thm:alg} provides formal guarantees on
the convergence rate and the
sample complexity from $\D$ needed
to estimate the expectations
sufficiently-accurately.

\paragraph{Do no harm.}
The distinction between our approach and most prior works on
fairness (especially \cite{kearns2017preventing})
is made clear from the ``do-no-harm''
property that \alg~exhibits, stated formally as Theorem~\ref{thm:donoharm}.
In a nutshell, the property guarantees that on any
subpopulation $S \subseteq \X$ that $\A$ audits,
the classification error cannot increase significantly
from $f_0$ to the post-processed classifier.
Further, the bound we prove is very pessimistic.
Both in theory and in practice, we do not expect
any increase to occur.  In particular,
the convergence analysis of \alg~follows
by showing that at every update, the average
cross-entropy loss on the population we update must
drop significantly.  Termination is guaranteed
because after too many iterations of auditing, the
post-processing will have learned $y$ perfectly.
Thus, if we use Algorithm~\ref{alg} to post-process a
model that is already achieves high accuracy on the validation distribution
the resulting model's accuracy should not deteriorate
in significant ways; empirically, we observe that
classification accuracy (on held-out test set)
tends to improve over $\D$ after multiaccuracy post-processing.

\subsection{Formal guarantees of \alg}
\label{sec:theory}
For clarity of presentation,
we describe the formal guarantees of
our algorithm assuming that $\A$ provably agnostic
learns a class of tests $\C$, in order to describe
the sample complexity appropriately.
The guarantees on the rate of convergence
do not rely on such an assumption.
We show that, indeed, Algorithm~\ref{alg} must
converge in a bounded number of iterations.
The proof follows by showing that, for an appropriately chosen $\eta$
(on the order of $\alpha$), each
update improves the cross-entropy loss over the updated set $S$, so the bound depends on
the initial cross-entropy loss.

To estimate the statistics used in \alg, we need to bound the sample complexity required for
the auditor to generalize.  Informally,
we say the \emph{dimension} $d(\C)$
of an agnostically learnable class $\C$
is a value such that
$m \ge \Omega\left(\frac{d(\C) + \log(1/\delta)}{\alpha^2}\right)$
random samples from $\D$ guarantee uniform convergence
over $\C$ with accuracy $\alpha$ with failure
probability at most $\delta$.  Examples of upper bounds on this notion of dimension
include $\log(\card{\C})$ for a finite class of tests,
the VC-dimension \cite{kearnsvazirani} for boolean tests,
and the metric entropy \cite{boucheron2013concentration} of
real-valued tests.  We state the formal guarantees as Theorem~\ref{thm:alg}.

\begin{theorem}\label{thm:alg}
Let $\alpha, \delta > 0$; suppose $\A$ agnostic learns
a class $\C \subseteq [-1,1]^\X$ of dimension $d(\C)$.
Then, using $\eta = O(\alpha)$,
Algorithm~\ref{alg} converges to a
$(\C,\alpha)$-multiaccurate hypothesis $f_T$ in
$T = O\left(\frac{\ell_\D(f_0;y)}{\alpha^2}\right)$
iterations from $m = \tilde{O}\left(T \cdot \frac{d(\C) + \log(1/\delta)}{\alpha^2}\right)$
random samples with probability
$\ge 1-\delta$.
\end{theorem}
Roughly speaking, for constant $\alpha,\delta$,
the sample complexity scales with the dimension
of the class $\C$.  For many relevant classes $\C$
for which we would want to enforce
$(\C,\alpha)$-multiaccuracy, this dimension will be
significantly smaller than the amount of data
required to train an accurate initial $f_0$.
Note also that our sample complexity is
completely generic and we make no effort to optimize
the exact bound. In particular, for structured $\C$
and $\A$, better uniform convergence bounds can be proved.
Further, appealing to a recent line of work on
adaptive data analysis initiated by
\cite{dwork2015reusable,russo2015much}, we can
avoid resampling at each iteration as in \cite{multi}.

\paragraph{Do no harm.}
The do-no-harm property guarantees that the classification
error on any subpopulation that $\A$ audits cannot
increase significantly.  As we assume $\A$ can identify
a very rich class of overlapping sets, in aggregate,
this property gives a strong guarantee on the utility of the
resulting predictor.  Further,
the proof of Theorem~\ref{thm:donoharm} reveals
that this worst-case bound is very pessimistic
and can be improved with stronger assumptions.
\begin{theorem}[Do-no-harm]\label{thm:donoharm}
Let $\alpha,\beta,\gamma > 0$
and $S \subseteq \X$ be a subpopulation where
$\Pr_{x\sim \D}[x \in S] \ge \gamma$.
Suppose
$\A$ audits the characteristic function $\chi_S(x)$ and its negation.
Let $f:\X\to [0,1]$ be the output of Algorithm~\ref{alg} when
given $f_0:\X \to [0,1]$, $\A$,
and $\alpha \le \beta\gamma$ as input.
Then the classification error of $f$ on the subset
$S$ is bounded as
\begin{equation}
\er_S(f;y) \le 3\cdot\er_S(f_0;y) + 4\beta.
\end{equation}
\end{theorem}

\paragraph{Derivative learning for faster convergence}
\label{subsec:derivative}
Here, we propose auditing with
an algorithm $\A_\ell$, described formally in
Algorithm~\ref{auditor} in the Appendix,
that aims to learn a smoothed version
of the partial derivative function
of the cross-entropy loss with respect to the
\emph{predictions}
$\frac{\partial \ell(f;y)}{\partial f(x)} = \frac{1}{1 - f(x) - y(x)}$, 
which grows in magnitude as $\card{f(x) - y(x)}$ grows.
We show that running \alg~ with
$\A_\ell$ converges
in a number of iterations that grows with $\log(1/\alpha)$,
instead of polynomially, as we would expect
for a smooth, strongly convex objective \cite{shalev2012online,bubeck2015convex}.
This sort of gradient method does
not typically make sense when we don't have information about $y(x)$ for all
$x \in \X$; nevertheless, if there is a
simple way to improve $f$,
we might hope to \emph{learn} the
partial derivative as a function of $x \in \X$ in order to update $f$.
This application of \alg~is similar in spirit to
gradient boosting techniques \cite{mason2000boosting,friedman2001greedy}, which interpret boosting algorithms as
running gradient descent on an appropriate cost-functional.

In principle, if the magnitude of the residual
$\card{f(x) -y(x)}$ is not too close to $1$ for
most $x \in \X$, then the learned partial derivative
function should correlate well with the true gradient.
Empirically, we observe
that $\A_\ell$ is effective at finding ways to
improve the model quite rapidly.
Formally, we show the following linear convergence
guarantee.
\begin{proposition}\label{prop:learning}
Let $\alpha,B, L > 0$ and $\C \subseteq [-B,B]^\X$.
Suppose we run Algorithm~\ref{alg} with $\eta = O(1/L)$
on initial model $f_0$ with auditor
$\A_\ell$ defined in Algorithm~\ref{auditor}.
Then, Algorithm~\ref{alg} converges in
$T = O\left(L \cdot \log(\ell_\D(f_0;y)/\alpha)\right)$
iterations.
\end{proposition}

\section{Experimental Evaluation}
\label{sec:experiments}

We evaluate the empirical performance of
\alg~ in three case studies.
The first and most in-depth case study aims to emulate the
conditions of the Gender Shades study \cite{ppb}, to test the
effectiveness of multiaccuracy auditing and post-processing on
this important real-world example.  In Section~\ref{sec:experiments:gender},
we show experimental results for auditing using two different
validation data sets.
In particular, one data set is fairly unbalanced and similar
to the data used to train, while the other data set was developed
in the Gender Shades study and is very balanced.
For each experiment, we report for various
subpopulations, the population percentage in $\D$,
accuracies of the initial model, our black-box
post-processed model, and white-box benchmarks.
In Section~\ref{sec:experiments:representation}, we discuss further subtleties of applying the multiaccuracy framework involving the representation of inputs passed to $\A$ for auditing;  in Section~\ref{sec:experiments:audit}, we show how auditing can be used beyond post-processing.  In particular, the hypotheses that $\A$ learns can be used to highlight subpopulations -- in an interpretable way --
on which the model is making mistakes.

We evaluate the effectiveness of multiaccuracy post-processing
on two other prediction tasks.  In both these case studies, we take the training and auditing
data distribution $\D$ to be the same, though the number of examples used for auditing will be quite small. Multiaccuracy still improves the performance
on significant subpopulations.  These results suggest some
interesting hypotheses about why machine-learned models may
incorporate biases in the first place, which warrant further
investigation.

\subsection{Multiaccuracy improves gender detection}
\label{sec:experiments:gender}
In this case study, we replicate the conditions of the Gender Shades study \cite{ppb}
to evaluate the effectiveness of the multiaccuracy framework in a realistic setting.
For our initial model, we train an inception-resnet-v1 \cite{szegedy2017inception} gender classification model using the CelebA data set with more than 200,000 face images \cite{celeba}. The resulting test accuracy on CelebA for binary gender classification is 98.4\%.

We applied \alg~ to this $f_0$ using two different auditing distributions.  In the first case, we audit using data from the LFW+a\footnote{We fixed the original data set's label noise for sex and race.} set \cite{lfwa, LFWTech}, which has similar demographic breakdowns as CelebA (i.e.~$\D \approx \D_0$).  In the second case, we audit using the PPB data set (developed in \cite{ppb}) which has balanced representation across gender and race (i.e.~$\D \neq \D_0$).
These experiments allows us to track the effectiveness of \alg~ as the representation of minority subpopulations changes.
In both cases, the auditor is ``blind'' -- it is not explicitly given the race or gender of any individual -- and knows nothing about the inner workings of the classifier.
Specifically, we take the auditor to be a variant of $\A_\ell$ (Algorithm~\ref{auditor}) that performs ridge regression to fit $\frac{\partial \ell_x(f;y)}{\partial f(x)} = \frac{1}{1-f(x)-y(x)}$.\footnote{To help
avoid outliers, we smooth the loss and use a
quadratic approximation for $\card{\frac{\partial \ell_x(f;y)}{\partial f(x)}} > 10$.} Instead of training the auditor on raw input pixels, we use the low dimensional representation of the input images derived by a variational autoencoder (VAE) trained on CelebA dataset using Facenet~\cite{faceNet} library.  (For more discussion of the representation used during auditing, cf.~Section~\ref{sec:experiments:representation}.)

To test the initial performance of $f_0$, we evaluated on a random subset of the LFW+a data containing 6,880 face images, each of which is labeled with both gender and race -- black (\textbf{B}) and non-black (\textbf{N}). 
For gender classification on LFW+a, $f_0$ achieves 94.4\% accuracy. Even though the overall accuracy is high, the error rate is much worse for females (23.1\%) compared to males (0.7\%) and worse for blacks (10.2\%) compared to non-blacks (5.1 \%); these results are qualitatively very similar to those observed by the commercial gender detection systems studied in \cite{ppb}.
We applied \alg, which converged in 7 iterations. The resulting classifier's classification error in minority subpopulations was substantially reduced, even though the auditing distribution was similar as the training distribution.

We compare \alg~ against a strong white-box baseline. Here, we retrain the network of $f_0$ using the audit set. 
Specifically, we retrain the last two layers of the network, which gives the best results amongst retraining methods. We emphasize that this baseline requires white-box access to $f_0$, which is often infeasible in practice. \alg~ accesses $f_0$ only as a black-box without any additional demographic information,
and still achieves comparable, if not improved, error rates compared to retraining.
We report the overall classification accuracy as well as accuracy on different subpopulations -- e.g.~\textbf{BF} indicates black female -- in Table~\ref{LFW_noinf}.
\begin{table}[h]
\centering
\begin{tabular}{c c c c c c c c c c}
 & \textbf{All} & \textbf{F} & \textbf{M} & \textbf{B} & \textbf{N} & \textbf{BF} & \textbf{BM} & \textbf{NF} & \textbf{NM} \\
\hline \hline
$\D$ & 100 & 21.0 & 79.0 & 4.9 & 95.1 & 2.1 & 18.8 & 2.7 & 76.3\\
\hline
$f_0$ & 5.4 & 23.1 & 0.7 & 10.2 & 5.1 & 20.4 & 2.1 & 23.4 & 0.6 \\
MA & 4.1 & 11.3 & 3.2 & 6.0 & 4.9 & 8.2 & 4.3 & 11.7 & 3.2\\
RT & 3.8 & 11.2 & 1.9 & 7.5 & 3.7 & 11.6 & 4.3 & 11.1 & 1.8

\end{tabular}
\caption{\textbf{Results for LFW+a gender classification.}
{\rm $\D$ denotes the percentages of each population in the data distribution; $f_0$ denotes the classification error (\%) of the initial predictor; MA denotes the classification error (\%) of the model after post-processing with \alg; RT denotes the classification error (\%) of the model after retraining on $\D$.}}
\label{LFW_noinf}
\end{table}

The second face dataset, PPB, in addition to being more balanced, is much smaller; thus, this experiment can be viewed as a stress test, evaluating the data efficiency of our post-processing technique.
The test set has 415 individuals and the audit set has size 855. PPB annotates each face as dark (\textbf{D}) or light-skinned (\textbf{L}).
As with LFW+a, we evaluated the test accuracy of the original $f_0$, the multiaccurate post-processed classifier, and retrained classifier on each subgroup.
\alg~converged in 5 iterations and again, substantially reduced error despite a small audit set and the lack of annotation about race or skin color (Table ~\ref{PPB_noinf}).

\begin{table}[ht]
\centering
\begin{tabular}{c c c c c c c c c c}

 &\textbf{All} & \textbf{F} & \textbf{M} & \textbf{D} & \textbf{L} & \textbf{DF} & \textbf{DM} & \textbf{LF} & \textbf{LM} \\
\hline \hline
$\D$ & 100 & 44.6& 55.4 & 46.4 & 53.6 & 21.4 & 25.0 & 23.2 & 30.4\\
\hline
$f_0$ & 9.9 & 21.6 & 0.4 & 18.8 & 2.2 & 39.8 & 1.0 & 5.2 & 0.0\\
MA & 3.9 & 6.5 & 1.8 & 7.3 & 0.9 & 12.5 & 2.9 & 1.0 & 0.8\\
RT & 2.2 & 3.8 & 0.9 & 4.2 & 0.4 & 6.8 & 1.9 & 1.0 & 0.0
\end{tabular}
\caption{\textbf{Results for the PPB gender classification data set.}
{\rm $\D$ denotes the percentages of each population in the data distribution; $f_0$ denotes the classification error (\%) of the initial predictor; MA denotes the classification error (\%) of the model after post-processing with \alg; RT denotes the classification error (\%) of the model after retraining on $\D$.}}
\label{PPB_noinf}
\end{table}

\begin{figure}[ht!]
\centering
\includegraphics[width=0.6\linewidth]{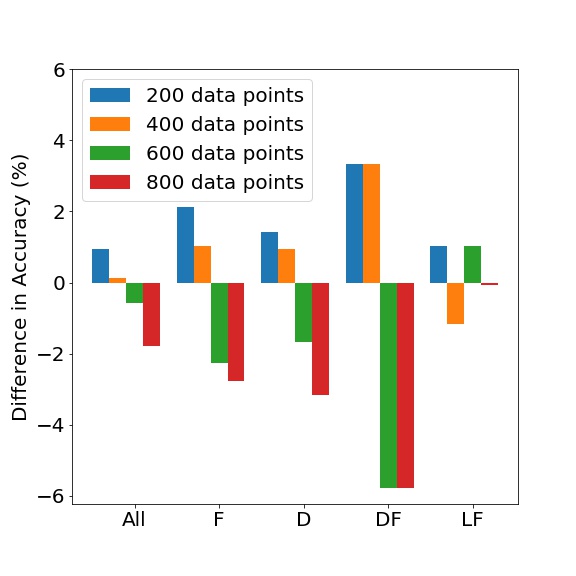}
\caption{\textbf{Multiaccuracy vs.~Retraining}: 
{\rm Difference in classification accuracy (i.e.~\% accuracy after \alg~ $-$ \% accuracy after retraining) is plotted on the vertical axis; this difference represents the accuracy advantage of ~\alg~ compared to retraining. As the size of the audit set shrinks, ~\alg~ has better performance both in overall accuracy and accuracy of the subgroups with the most initial bias because it is more data efficient.}
\label{fig:MA_va_Re}}
\end{figure}

To further test the data efficiency of ~\alg, we evaluate the effect of audit set size on the resulting accuracy of each method.
In Fig.~\ref{fig:MA_va_Re}, we report the performance of ~\alg~ versus the white-box retraining method for different sizes of audit set.
The plot displays the difference in accuracy for the overall population along with the subgroups that suffered the most initial bias.
It shows that the performance of ~\alg~ may actually be better than the white-box retraining baseline when 
validation data is especially scarce.

\subsubsection{Representation matters}
\label{sec:experiments:representation}
As discussed earlier, in the reported gender detection experiments,
we audit for multiaccuracy using ridge regression over an encoding of images produced by a variational autoencoder.
Using the representation of images produced by this encoding intuitively makes sense, as the autoencoder's reconstruction objective aims to preserve as much information about the image as possible while reducing the dimension.
Still, we may wonder whether multiaccuracy auditing over a different representation of the images would perform better.
In particular, since we are interested in improving the accuracy on the gender detection task, it seems plausible that a representation of the images based on the internal layers of the initial prediction network might preserve the information salient to gender detection more effectively.

We investigate the importance of the representation used to audit empirically.
In particular, we also evaluate the performance of \alg~ using the same auditor $\A_\ell$ run over two other sets of features, given by the last-layer and the second-to-last layer of the initial prediction residual network $f_0$.
In Table~\ref{table:representation}, we show that using the unsupervised VAE representation yields the best results.
Still, the representations from the last and second-to-last layers are
competitive with that of the VAE, and in some subpopulations are better.

Collectively, these findings bolster the argument for ``fairness through awareness'', which advocates that in order to make fair predictions, sensitive information (like race or gender) should be given to the (trustworthy) classifier.
While none of these representations explicitly encode sensitive group information, the VAE representation does preserve information about the original input, for instance skin color, that seems useful in understanding the group status.
The prediction network is trained to have the best prediction accuracy (on an unbalanced training data set), and thus,
the representations from the network reasonably may contain less information about these sensitive features.
These results suggest that the effectiveness of multiaccuracy does depend on the representation of inputs used for auditing, but so long as the representation is sufficiently expressive, \alg~ may be robust to the exact encoding of the features.

\begin{table}[h]
\centering
\begin{tabular}{c c c c c c c c c c}

 &\textbf{All} & \textbf{F} & \textbf{M} & \textbf{D} & \textbf{L} & \textbf{DF} & \textbf{DM} & \textbf{LF} & \textbf{LM} \\
\hline \hline
{\bf LFW+a:}\\
VAE & 4.1 & 11.3 & 3.2 & 6.0 & 4.9 & 8.2 & 4.3 & 11.7 & 3.2\\
$R_{1, f_0}$ & 4.9 & 13.6 & 2.6 & 6.3 & 4.9 & 8.8 & 4.3 & 14.1 & 2.6\\
$R_{2, f_0}$ & 4.5 & 12.6 & 2.4 & 6.3 & 4.4 & 8.8 & 4.3 & 13.1 & 2.3\\
\hline
{\bf PPB:}\\
VAE & 3.9 & 6.5 & 1.8 & 7.3 & 0.9 & 12.5 & 2.9 & 1.0 & 0.8\\
$R_{1,f_0}$ & 4.3 & 7.6 & 1.7 & 7.8 & 1.3 & 13.6 & 2.9 & 2.1 & 0.8\\
$R_{2,f_0}$ & 5.1 & 9.7 & 1.3 & 9.4 & 1.3 & 17.0 & 2.9 & 3.1 & 0.0\\
\end{tabular}
\caption{\textbf{Effect of representation on the \alg~ performance}
{\rm  VAE denotes the denotes the classification error (\%) using the VAE representation; $R_{1, f_0}$ denotes the classification error (\%) using the classifier's last layer representation, $R_{2, f_0}$ denotes the classification error (\%) using the classifier's second to last layer representation} }
\label{table:representation}
\end{table}

\subsubsection{Multiaccuracy auditing as diagnostic}
\label{sec:experiments:audit}

As was shown in \cite{ppb}, we've demonstrated that models trained in good faith on unbalanced data may exhibit significant biases on
the minority populations.  For instance, the initial classification error on black females is significant, whereas on white males, it is near $0$.
Importantly, the only way we were able to report these accuracy
disparities was by having access to a rich data set where gender
and race were labeled.
Often, this demographic information will not be available;
indeed, the CelebA images are not labeled with race information,
and as such, we were unable to evaluate the
subpopulation classification accuracy on this set.
Thus, practitioners may be faced with a problem:
even if they know their model is making undesirable mistakes,
it may not be clear if these mistakes are concentrated on
specific subpopulations.
Absent any identification of the subpopulations on which
the model is underperforming, collecting additional training data
may not actually improve performance across the board.

We demonstrate that multiaccuracy auditing may serve as
an effective diagnostic and interpretation tool to help developers identify systematic biases in their models.
The idea is simple: the auditor returns a
hypothesis $h$ that essentially ``scores'' individual inputs
$x$ by how wrong the prediction $f_0(x)$ is. If we consider the magnitude of their scores $\card{h(x)}$, then we may understand better the
biases that the encoder is discovering.

As an example, we test this idea on the PPB data set, evaluating the test images'
representations with the hypotheses the auditor returns.
In Figure~\ref{fig:discovery}, we display the images in the test set that get the highest and lowest effect ($\card{h(x)}$ large and $\card{h(x)} \approx 0$, respectively) according to the first and second hypothesis returned by $\A_\ell$.
In the first round of auditing, the three highest-scoring images (top-left row) are all women, both black and white.
Interestingly, all of the least active images (bottom-left row) are men in suits, suggesting that suits may be a highly predictive feature of being a man according to the original classifier, $f_0$. Overall the first round of audit seems to primarily identify gender as the axis of bias in $f_0$. In the second round, after the classifier has been improved by one step of \alg, the auditor seems to hone in on the  ``dark-skinned women'' subpopulation as the region of bias, as the highest activating images (top-right row) are all dark-skinned women.

\begin{figure}[ht!]
\centering
\includegraphics[width=0.48\linewidth]{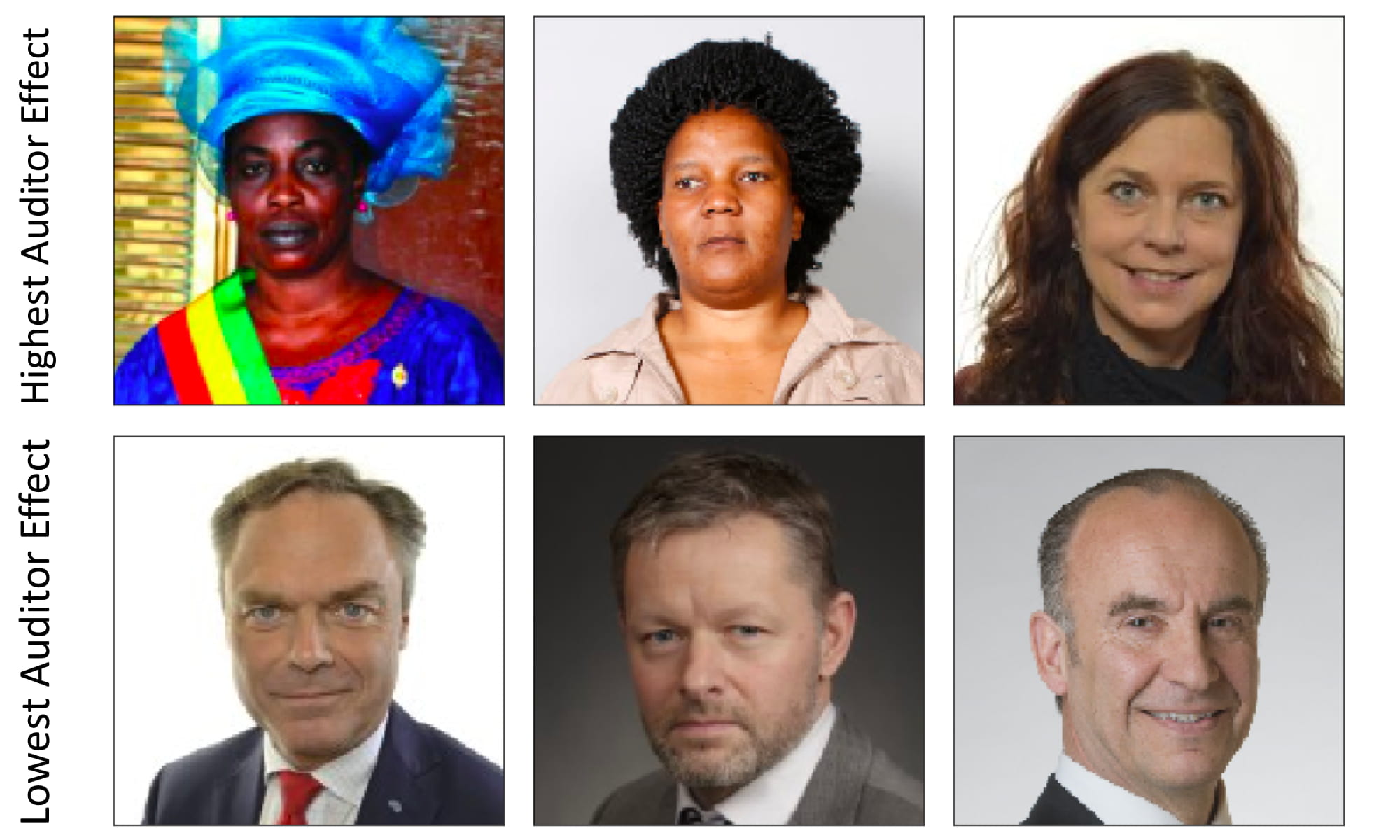}
\includegraphics[width=0.48\linewidth]{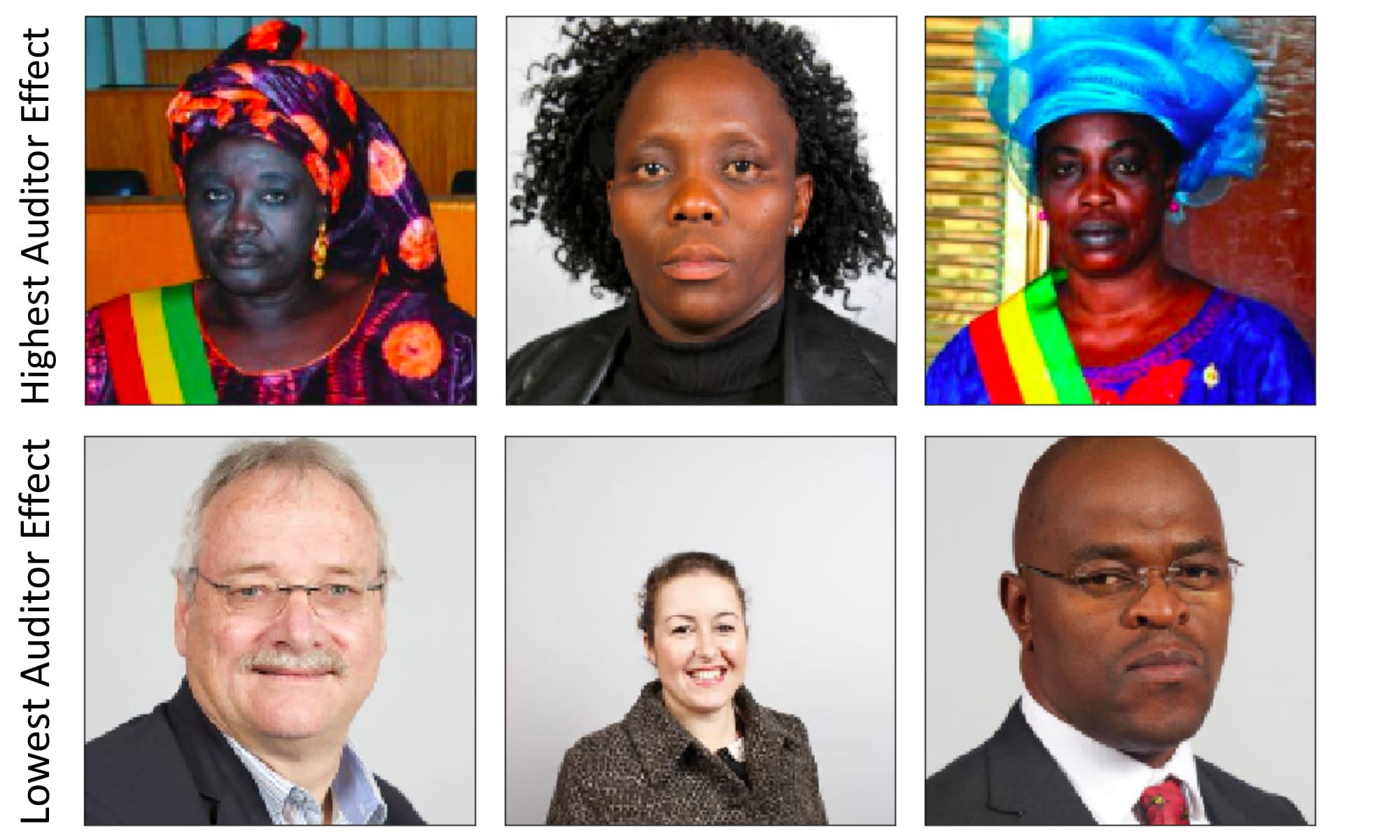}
\caption{\textbf{Interpreting Auditors} 
{\rm Here, we depict the PPB test images with the highest and lowest activation of the first and second trained auditor. The images with the highest auditor effects corresponds to images where the auditor detects the largest bias in the classifier. In the first round of auditing, the most biased images are women, both black and white. In the second round of auditing, after the base classifier has been augmented by one step of \alg, the most biased images are more specifically dark-skinned women.}
\label{fig:discovery}}
\end{figure}

\subsection{Additional case studies}
Multiaccuracy auditing and post-processing is applicable broadly in supervised learning tasks, not just in image classification applications.
We demonstrate the effectiveness of \alg~ in two other settings:
the adult income prediction task and a semi-synthetic disease prediction task.

\paragraph{Adult Income Prediction}
For the first case study, we utilize the adult income prediction
data set ~\cite{adult} with 45,222 samples and 14 attributes (after removing subjects with unknown attributes) for the
task of binary prediction of income more than \$50k for the two major groups of Black and White.  We remove the sensitive features of gender -- female (\textbf{F}) and male (\textbf{M}) and race (for the two major groups) -- black (\textbf{B}) and white (\textbf{W}) -- from the data, to simulate settings where sensitive features are not available to the algorithm training.
We trained a base algorithm, $f_0$, which is a neural network with two hidden layers on 27,145 randomly selected individuals. The test set consists of an independent set of 15,060 persons.

We audit using a decision tree regression model
(max depth $5$) $\A_{\text{dt}}$ to fit the residual $f(x) - y(x)$.  $\A_{\text{dt}}$ receives samples of
validation data drawn from the same distribution as training;
that is $\D = \D_0$.
In particular, we post-process with 3,017 individuals
sampled from the same adult income dataset (disjoint from the training set of $f_0$).
The auditor is given the same features as the original prediction model, and thus, is not given the gender or race of any individual.
We evaluate the post-processed classifier on the same independent test set.  
\alg~ converges in 50 iterations with $\eta = 1$.

As a baseline, we trained four separate  neural networks with the same architecture as before (two hidden layers) for each of the four subgroups using the audit data.
As shown in Table~\ref{table:adult}, multiaccuracy post-processing achieves better accuracy both in aggregate and for each of the subgroups.
Importantly,
the subgroup-specific models requires explicit access to the sensitive features of gender and race.
Training a classifier for each subgroup, or explicitly adding subgroup accuracy into the training objective, assumes that the subgroup is already identified in the data. This is not feasible in the many applications where, say, race or more granular categories are not given.
Even when the subgroups are identified, we often do not have enough samples to train accurate classifiers on each subgroup separately.
This example illustrates that multiaccuracy can help to boost the overall accuracy of a black-box predictor in a data efficient manner.

\begin{table}[h]
\centering
\begin{tabular}{c c c c c c c c c c c c}

 &\textbf{All} & \textbf{F} & \textbf{M} & \textbf{B} & \textbf{W} & \textbf{BF} & \textbf{BM} & \textbf{WF} & \textbf{WM} \\
\hline \hline
$\D$ & 100.0 & 32.3 & 67.7 & 90.3 & 9.7 & 4.8 & 4.9 & 27.4 & 62.9\\
\hline
$f_0$ & 19.3 & 9.3 & 24.2 & 10.5 & 20.3 & 4.8 & 15.8 & 9.8 & 24.9 \\
MA & 14.7 & 7.2 & 18.3 & 9.4 & 15.0 & 4.5 & 13.9 & 7.3 & 18.3 \\
SS & 19.7 & 9.5 & 24.6 & 10.5 & 19.9 & 5.5 & 15.3 & 10.2 & 25.3\\
\end{tabular}
\caption{\textbf{Results for Adult Income Data Set}
{\rm $\D$ denotes the percentages of each population in the data distribution; $f_0$ denotes the classification error (\%) of the initial predictor; MA denotes the classification error (\%) of the model after post-processing with \alg; SS denotes the classification error (\%) of the subgroup-specific models trained separately for each population.}}
\label{table:adult}
\end{table}

\subsubsection{Semi-Synthetic Disease Prediction}
We design a disease prediction task based on real individuals, where the phenotype to disease relation is designed to be different for different subgroups, in order to simulate a challenging setting. 
We used 40,000 individuals sampled from the UK Biobank~\cite{biobank}. Each individual contains 60 phenotype features. To generate a synthetic disease outcome for each subgroup, we divided the data set into four groups based on gender -- male (\textbf{M}) and female (\textbf{F}) -- and age -- young (\textbf{Y}) and old (\textbf{O}). For each subgroup, we create synthetic binary labels using a different polynomial function of the input features with different levels of difficulty. The polynomial function orders are 1, 4, 2, and 6 for OF, OM, YF, and YM subgroups respectively.

For $f_0$, we trained a  neural network with two hidden layers on 32,000 individuals, without using the gender and age features. Hyperparameter search was done for the best weight-decay and drop-out parameters.
The $f_0$ we discover performs moderately well on every subpopulation, with the exception of old females (\textbf{OF}) where the classification error is significantly higher.  Note that this subpopulation had the least representation in $\D_0$.  Again, we audit using $\A_{\text{dt}}$
to run decision tree regression with validation data samples drawn from $\D = \D_0$.  Specifically, the auditor receives a sample of
4,000 individuals without the gender or age features.
As a baseline, we trained a separate classifier for each of
the subgroups using the same audit data.
As Table~\ref{table:ukb} shows, \alg~ significantly lowers the classification error in the old female population.

\begin{table}[h]
\centering
\begin{tabular}{ c c c c c c c c c c c c c}

&\textbf{All} & \textbf{F} & \textbf{M} & \textbf{O} & \textbf{Y} & \textbf{OF} & \textbf{OM} & \textbf{YF} & \textbf{YM}\\
\hline \hline
$\D$ & 100 & 39.6 & 60.4 & 34.6 & 65.4 & 15.0 & 19.7 & 24.6 & 40.7\\
\hline
$f_0$ & 18.9 & 29.4 & 12.2 & 21.9 & 17.3 & 36.8 & 10.9 & 24.9 & 12.8\\
MA &  16.0 & 24.1 & 10.7 & 16.4 & 15.7 & 26.5 & 9.0 & 22.7 & 11.6\\
SS & 19.5 &	32.4 & 11.0  & 22.1 &	18.1 & 37.6 & 10.3 & 29.3 & 11.3\\

\end{tabular}
\caption{\textbf{Results for UK Biobank semi-synthetic data set.}
{\rm $\D$ denotes the percentages of each population in the data distribution; $f_0$ denotes the classification error (\%) of the initial predictor; MA denotes the classification error (\%) of the model after post-processing with \alg; SS denotes the classification error (\%) of the subgroup-specific models trained separately for each population.}}
\label{table:ukb}
\end{table}

\section{Discussion}
In this work, we propose multiaccuracy auditing and post-processing as a method for improving the fairness and accountability of black-box prediction systems.
Here, we discuss how our work compares to prior works, specifically, how it fits into the growing literature on fairness for learning systems.  We conclude with further discussion of our results and speculation about future investigations.

\subsection{Related works}
Many different notions of fairness have been proposed in literature on
learning and classification
\cite{fta,hardt2016equality,zemel2013learning,dwork2017decoupled,multi,kearns2017preventing,tatsu,ftba,RothlbumY18}. Many of these works encode some notion of parity, e.g.~different subgroups should have similar false positive rates,
as an explicit objective/constraint in the training of the original classifier.  The fairness properties are viewed as
constraints on the classifier that ultimately
\emph{limit the model's utility}.  A common belief is
that in order to achieve equitable treatment for
protected subpopulations, the performance on other
subpopulations necessarily must degrade.

A notable exception to this pattern is the work of H\'{e}bert-Johnson \emph{et al.}~\cite{multi},
which introduced a framework for achieving fairness notions that
aim to provide accurate predictions for many important subpopulations.  
\cite{multi} introduced the
notion of \emph{multiaccuracy}\footnote{\cite{multi} refers to this notion as ``multi-accuracy-in-expectation''.} and a stronger notion, dubbed
\emph{multicalibration}, in the context of regression tasks.
Multicalibration guarantees (approximately) calibrated
predictions, not just overall, but on a rich class of structured ``identifiable" subpopulations.
\cite{multi} provides theoretical algorithms for achieving multiaccuracy and
multicalibration, and shows how to post-process a model to achieve
multicalibration
in a way that \emph{improves} the regression objective
across all subpopulations (in terms of squared-error).
Our work directly extends the approach of \cite{multi},
adapting their work to the binary
classification setting.
Our post-processing algorithm, \alg, builds on the algorithm
given in \cite{multi}, providing the additional ``do-no-harm''
property. This property guarantees that if the initial predictor $f_0$
has small classification error on some identifiable group, then
the resulting post-processed model will also have small classification
error on this group.

Independent work of
Kearns~\emph{et al.} \cite{kearns2017preventing} also
investigated how to achieve statistical fairness
guarantees, not just for traditionally-protected
groups, but on rich families of subpopulations.
\cite{kearns2017preventing} proposed
a framework for \emph{auditing} and learning models to achieve
fairness notions like statistical parity and
equal false positive rates.
Both works \cite{multi,kearns2017preventing} connect the task of learning a model that
satisfies the notion of fairness to the task of (weak)
agnostic learning \cite{kearns,agnostic,kalai2008agnostic,feldman2010distribution}.  \cite{kearns2017preventing} also reduces the problem of learning a classifier satisfying parity-based notions of fairness across subgroups to the problem of auditing; it would be interesting if their notion of auditing can be used by humans as a way to diagnose systematic discrimination.

Our approach to post-processing, which uses a learning algorithm
as a fairness auditor, is similar in spirit to the approach to learning
of \cite{kearns2017preventing}, but differs
technically in important ways.
In particular, in the framework of \cite{kearns2017preventing},
the auditor is used during (white-box) training to \emph{constrain} the model
selected from a pre-specified hypothesis class; ultimately,
this constrains the accuracy of the predictions.  In our
setting (as in \cite{multi}), we do not restrict ourselves to an
explicitly-defined hypothesis class, so we can augment the current model using the auditor; these
augmentations \emph{improve} the accuracy of the model.

Indeed, at a technical level, our post-processing algorithm is most similar to work on boosting \cite{freund1995desicion,schapire2012boosting}, specifically, gradient boosting \cite{mason2000boosting,friedman2001greedy}.  Still, our perspective is quite different from the standard boosting setting.
Rather than using an expressive class of predictors as the base classifiers to be able to learn the function directly, our setting focuses on the regime where data is limited and we must restrict our attention to simple classes.  Thus, it becomes important that we leverage the expressiveness (and initial accuracy) of $f_0$ if we are to obtain strong performance using the multiaccuracy approach.  Further, the termination of \alg~ certifies that the final model satisfies $(\A,\alpha)$-multiaccuracy; in general, standard boosting algorithms will not provide such a certificate.

Motivated by unfairness that arises as the result of feedback loops in classification settings, another recent work of Hashimoto \emph{et al.}~\cite{tatsu}
aims to improve fairness at a subpopulation level.  Specifically, their notion of fairness similarly aims to give accurate (i.e. bounded loss) predictions not just overall, but on \emph{all}
significant subpopulations.
In the multiaccuracy setting, we argued that this goal was information-theoretically infeasible; \cite{tatsu}
sidesteps this impossibility by optimizing over a fixed hypothesis class, and
formulating the problem as a min-max optimization.
They give show how to relax the problem of minimizing the worst-case
subpopulation loss and reduce the relaxation to a certain robust optimization problem.
While their approach does not guarantee optimality, it gives a strong certificate upper-bounding the maximum loss over all subpopulations.

A different approach to subgroup fairness is studied by Dwork~\emph{et al.}
\cite{dwork2017decoupled}.  This work investigates the
question of how to learn a ``decoupled'' classifier, where separate
classifiers are learned for each subgroup and then combined to achieve
a desired notion of fairness.
While applicable in some settings, at times,
this approach may be untenable.
First, decoupling the classification problem requires that we have race, age, and other attributes of interest in the dataset and that the groups we wish to protect are partitioned by these attributes; this information is often not available.
Even if this information is available, \emph{a priori},
it may not always be obvious which subpopulations
require special attention.
In contrast, the multiaccuracy approach allows
us to protect a rich class of overlapping subpopulations without explicit knowledge
of the vulnerable populations. An interesting direction for future investigation could try to pair multiaccuracy auditing (to identify subpopulations in need of protection) with the decoupled classification techniques of \cite{dwork2017decoupled}.

The present work, along with \cite{multi,kearns2017preventing,ftba},
can be viewed as studying information-fairness tradeoffs in prediction
tasks (i.e.~strengthening the notion of fairness that can be guaranteed
using a small sample).  These works fit into the larger literature
on fairness in learning and prediction tasks
\cite{fta,zemel2013learning,ppb,hardt2016equality,dwork2017decoupled,ftba,RothlbumY18}, discussions of the utility-fairness tradeoffs in
fair classification \cite{propublica,kleinberg2016inherent,chouldechova2017fair,chouldechova2017fairer,corbett2017algorithmic,pleiss2017fairness}.
While fairness and accountability serve as the main motivations for developing the multiaccuracy framework, our results may have broader interest.
In particular, multiaccuracy post-processing may be applicable in domain adaptation settings, particularly under label distribution shift as studied recently in \cite{liptonWS18}, but when the learner gets a small number of labeled samples from the new distribution.

\subsection{Conclusion}

The multiaccuracy framework can be applied very broadly; importantly, we can post-process any initial model $f_0$ given only black-box access to $f_0$ and a small set of labeled validation data.
We show that in a wide range of realistic settings, post-processing for multiaccuracy helps to mitigate systematic biases in predictors across sensitive subpopulations, even when the identifiers for these subpopulations are not given to the auditor explicitly.
In our experiments, we observe that standard supervised learning optimizes for overall performance, leading to settings where certain subpopulations incur substantially worse error rates.
Multiaccuracy provides a framework for fairness in classification by improving the accuracy in identifiable subgroups,
in a way that suffers no tradeoff between accuracy and utility.  
We demonstrate -- both theoretically
and empirically -- that post-processing with \alg~serves
as an effective tool for improving the
accuracy across important subpopulations,
and does not harm the populations
that are already classified well.

Multiaccuracy works to the extent that the auditor can effectively identify specific subgroups where the original classifier $f_0$ tends to make mistakes. The power of multiaccuracy lies in the fact that in many settings, we can identify issues with $f_0$ using a relatively small amount of audit data.
Thus, multiaccuracy auditing is limited: if the mistakes
appear overly-complicated to the bounded auditor
(for information- or complexity-theoretic reasons),
then the auditor will not be able to identify these mistakes.
Our empirical results suggest, however, that in many
realistic settings, the subpopulations on which a
classifier errs are efficiently-identifiable.
This observation may be of interest beyond the
context of fairness.
In particular, our experiments
improving the accuracy of a model trained on
CelebA on the LFW+a and PPB test sets suggests
a lightweight black-box alternative to more
sophisticated transfer learning techniques,
which may warrant further investigation.

Our empirical investigations reveal some additional interesting
aspects of the multiaccuracy
framework.  In particular, we've shown that multiaccuracy auditing can identify
underrepresented groups receiving suboptimal predictions
even when the sensitive attributes
defining these groups are not explicitly given to the auditor, which proves useful for diagnosing where models make mistakes. We feel that it may be of further interest within the study of model interpretability.
Finally, it is striking that \alg~ tends to improve, not just subgroup accuracy, but also the overall accuracy, even when
the minority groups
remain underrepresented in the validation data.
While some of these findings may be due to suboptimal training of our initial models,
we believe this is not the only factor at play.  In particular, we hypothesize that understanding why models incorporate biases during training -- and further, why simple interventions like multiaccuracy post-processing can significantly improve generalization error --
requires investigating the dynamics of overfitting during training, not just on the population as a whole, but across significant subpopulations.

\paragraph{Acknowledgements.}
The authors thank Omer Reingold and Guy N.~Rothblum for their
advice and
helpful discussions throughout the development of this work;
we thank Weihao Kong, Aditi Raghunathan, and Vatsal Sharan for
feedback on early drafts of this work.

\bibliographystyle{alpha}
\bibliography{refs}

\clearpage
\appendix

\paragraph{Appendix notation.}
We use the inner product
$\langle h,g \rangle = \E_{x \sim \D}[h(x)\cdot g(x)]$
and the $p$-norms
$\norm{h}_p = \left(\E_{x\sim \D}[\card{h(x)}^p]\right)^{1/p}$.

\section{Multiaccuracy and classification error}
Here, we prove Proposition~\ref{prop:ac2er}.
\begin{proposition*}[Restatement of Propostion~\ref{prop:ac2er}]
Let $\hat{y}:\X \to \set{-1,1}$ as $\hat{y}(x)=1-2y(x)$.
Suppose that for $S \subseteq \X$ with $\Pr_{x \sim \D}[x \in S] \ge \gamma$,
there is some $c \in \C$
such that $\norm{c - \hat{y}_S}_1 \le \tau$.
Then if $f$ is $(\C,\alpha)$-multiaccurate,
$\er_S(f;y) \le 2\cdot(\alpha + \tau)/\gamma$.
\end{proposition*}

\begin{proof}
For $i,j \in \set{0,1}$,
let $S_{ij} = \set{x \in S : y(x) = i \land \bar{f}(x) = j}$.
Further denote $\beta_{ij} = \Pr_{x \sim \D}[x \in S_{ij}]$.
Note that the classification error on a set $S$ is
$\er_S(f;y) \le (\beta_{01} + \beta_{10})/\gamma$.

Let $\hat{y}(x) = 1-2y(x)$ and suppose
$c(x) = \hat{y}(x)_S + z(x)$ where $\norm{\delta}_1 \le \tau$.
Then, we derive the following inequality.
\begin{align}
&\quad \E_{x \sim \D}[c(x)\cdot (f(x) - y(x))]\\
&= \E_{x \sim \D}[\hat{y}(x)_S\cdot (f(x)-y(x))] +
\E_{x \sim \D}[z(x)\cdot (f(x)-y(x))]\\
&\ge \beta_{01}\cdot \E_{x \sim S_{01}}[f(x)-y(x)]
+ \beta_{10}\cdot \E_{x \sim S_{10}}[y(x) - f(x)] - \tau\label{ineq:0}
\end{align}
where (\ref{ineq:0}) follows by H\"{o}lder's inequality,
from the fact that the contribution to the expectation
of $(1-2y(x))\cdot(f(x) - y(x))$
from $S_{00}$ and $S_{11}$ is lower bounded by $0$, and
by the definition $\hat{y}_S(x) = 0$
for $x \not \in S$.
Further, because we know any $x \in S_{01} \cup S_{10}$ is misclassified,
we can lower bound the contribution by $1/2$.
Thus, if $\E_{x \sim \D}[c(x) \cdot (f(x)-y(x))] \le \alpha$,
then by rearranging we conclude
\begin{equation}
\er_S(f;y) = (\beta_{01} + \beta_{10})/\gamma \le 2\cdot(\alpha + \tau)/\gamma.
\end{equation}
\end{proof}

Theorem~\ref{thm:donoharm} follows by a similar argument.
\begin{theorem*}[Restatement of Theorem~\ref{thm:donoharm}]
Let $\alpha,\beta,\gamma > 0$
and $S \subseteq \X$ be a subpopulation where
$\Pr_{x\sim \D}[x \in S] \ge \gamma$.
Suppose for 
$\A$ audits the characteristic function $\chi_S(x)$ and its negation.
Let $f:\X\to [0,1]$ be the output of Algorithm~\ref{alg} when
given $f_0:\X \to [0,1]$, $\A$,
and $0 < \alpha \le \beta\gamma$ as input.
Then the classification error of $f$ on the subset
$S$ is bounded as
\begin{equation}
\er_S(f;y) \le 3\cdot\er_S(f_0;y) + 4\beta.
\end{equation}
\end{theorem*}
\begin{proof}
Suppose that $\er_S(f_0;y) \le \tau$.
Consider $S_1 = \set{x \in S: f_0(x) > 1/2}$;
suppose $\er_{S_1}(f_0;y) = \tau_1$.
By assumption, $-\chi_S(x)$ is audited on $\X_1$.
Consider
$\E_{x \sim S_1}[-\chi_S(x) \cdot (f(x) - y(x))]$.
\begin{align}
&\quad \E_{x \sim S_1}[-\chi_S(x)\cdot (f(x) - y(x))]\\
&=
\E_{x \sim S_1}[y(x) - f(x)]\\
&= \Pr_{x \sim S_1}[y(x) = 1]\cdot \E_{\substack{x \sim S_1\\y(x) = 1}}[1 - f(x)]
- \Pr_{x \sim S_1}[y(x) = 0]\cdot \E_{\substack{x\sim S_1\\y(x) = 0}}[f(x)]\\
&\ge \Pr_{x \sim S_1}[y(x) = 1 \land \bar{f}(x) = 0]
\cdot \E_{\substack{x \sim S_1\\y(x) = 1\\ 
\bar{f}(x) = 0}}[1 - f(x)] - \tau_1 \label{ineq:donoharm:1}\\
&\ge \frac{1}{2}\Pr_{x \sim S_1}[y(x) = 1 \land \bar{f}(x) = 0] - \tau_1 \label{ineq:donoharm:2}
\end{align}
where (\ref{ineq:donoharm:1}) follows from applying
H\"older's inequality and the assumption that
$\er_{S_1}(f_0;y) = \tau_1$; and (\ref{ineq:donoharm:2})
follows from lower bounding the contribution to the
expectation based on the true label and the predicted
label.
Note that $\Pr_{x \sim S}[x \in S_1] \cdot
\E_{x \sim S_1}[y(x) - f(x)] \le \alpha/\gamma = \beta$
by the fact that $f$ passes multiaccuracy auditing
by $\A$ and the assumption that $\Pr_{x\sim \D}[x \in S]
\ge \gamma$.  Rearranging gives the following inequality
\begin{align}
\er_{S_1}(f;y) &\le
\frac{2\beta}{\Pr_{x \sim S}[x \in S_1]} + 3\tau_1
\end{align}
where the additional $\tau_1$ comes from
accounting for the false positives.

A similar argument holds for $S_0$
with $\er_{S_0}(f_0;y) = \tau_0$,
using $\chi_S(x)$.
We can expand $\er_{S}(f;y)$ as a convex combination
of the classification error over $S_0$ and $S_1$.
\begin{align}
&~~\quad \er_S(f;y)\\ &= \Pr_{x \sim S}[x \in S_0]\cdot \er_{S_0}(f;y) 
+ \Pr_{x \sim S}[x \in S_1] \cdot \er_{S_1}(f;y)\\
&\le \Pr_{x \sim S}[x \in S_0]\cdot
\Pr_{x \sim S_0}[y(x) \neq \bar{f}(x)]
+ \Pr_{x \sim S}[x \in S_1]\cdot
\Pr_{x \sim S_1}[y(x) \neq \bar{f}(x)]\\
&\le \Pr_{x \sim S}[x \in S_0]\cdot
\left(3\tau_0 + \frac{2\beta}{\Pr_{x \sim S}[x \in S_0]}\right)\notag\\
&\qquad\qquad\qquad\qquad + \Pr_{x \sim S}[x \in S_1]\cdot
\left(3\tau_1 + \frac{2\beta}{\Pr_{x \sim S}[x \in S_1]}\right)\\
&=3\cdot\left(\Pr_{x\sim S}[x \in S_0]\cdot \tau_0
+ \Pr_{x \sim S}[x \in S_1]\cdot \tau_1\right) + 4 \beta\\
&\le 3\tau + 4\beta
\end{align}
by the fact that $S$ is partitioned into $S_0$ and
$S_1$ and $\tau$ is a corresponding convex combination of
$\tau_0$ and $\tau_1$.
\end{proof}

\section{Analysis of Algorithm~\ref{alg}}
\label{app:mw}
Here, we analyze the sample complexity and running
time of Algorithm~\ref{alg}.
\begin{theorem*}[Restatement of Theorem~\ref{thm:alg}]
Let $\alpha, \delta > 0$ and suppose $\A$ agnostic learns
a class $\C \subseteq [-1,1]^\X$ of dimension $d(\C)$.
Then, using $\eta = O(\alpha)$,
Algorithm~\ref{alg} converges to a
$(\C,\alpha)$-multiaccurate hypothesis $f_T$ in
$T = O\left(\frac{\ell_\D(f_0;y)}{\alpha^2}\right)$
iterations from $m = \tilde{O}\left(T \cdot \frac{d(\C) + \log(1/\delta)}{\alpha^2}\right)$
samples with probability at least $1-\delta$ over the random samples.
\end{theorem*}

\subsection{Sample complexity}
We essentially assume the sample complexity issues away by working
with the notion of dimension.  We give an example proof outline of a standard
uniform convergence argument using metric entropy
as in \cite{boucheron2013concentration}.
\begin{lemma}
Suppose $\C \subseteq [-1,1]^\X$ has $\eps$-covering number $N_\eps = {\cal N}(\eps,\C,\norm{\cdot}_1)$.
Then, with probability at least $1-\delta$,
\begin{equation}
\card{\frac{1}{m}\sum_{i =1}^m\left(c(x_i)y(x_i)\right) - 
\E_{x \sim \D}[c(x)y(x)]} \le O\left(\alpha\right)
\end{equation}
provided $m \ge \tilde{\Omega}\left(\frac{\log(N_{\Theta(\alpha)}/\delta)}{\alpha^2}\right)$.
\end{lemma}
\begin{proof}
The lemma follows from a standard uniform convergence argument.  First, observe that
because every $c:\X \to [-1,1]$ and $y \in \set{0,1}$
that the empirical estimate using $m$ samples has
sensitivity $1/m$.  Thus, we can apply McDiarmid's inequality to show concentration
of the following statistic.
\begin{equation}
\sup_{c \in \C}\card{\frac{1}{m}\sum_{i =1}^m\left(c(x_i)y(x_i)\right) - 
\E_{x \sim \X}[c(x)y(x)]}
\end{equation}
Then, using a standard covering argument,
for $N = {\cal N}(\eps,\C,\norm{\cdot}_1)$
the $\eps$-covering number, we can bound the
deviation with high probability.
Specifically, taking
$O\left(\frac{\log(N/\delta)}{\alpha^2}\right)$ samples guarantees
that the empirical estimate for each $c \in \C$
will be within $O(\alpha)$ with probability at least $1-\delta$.
Taking $\delta$ small enough to union bound against every iteration and
adjusting constants shows gives the lemma.
\end{proof}
Note that this analysis is completely generic,
and more sophisticated arguments may improve the
resulting bounds that leverage structure in the
specific $\C$ of interest.

\subsection{Convergence analysis}

We will track progress of Algorithm~\ref{alg} by tracking the expected cross-entropy
loss.  We show that every update makes the
expected cross-entropy loss decrease significantly.
As the loss is bounded below by $0$, then
positive progress at each iteration combined
with an upper bound on the initial loss gives the
convergence result.

Note that when we estimate the statistical queries from data,
we only have access to approximate answers.
Thus, per the sample complexity argument above,
we assume that each statistical query
is $\alpha/4$-accurate.  Further, we will update $f_t$ if we find an update $c_t$ where
$\langle c_t , f-y \rangle \ge 3\alpha/4$.
Thus, at convergence, it should be clear that the resulting
hypothesis will be $(\C,\alpha)$-multiaccurate.
The goal is to show that this way, \alg~converges quickly.
\begin{lemma}
Let $\alpha > 0$ and suppose
$\C \subseteq [-1,1]^\X$.
Given access to statistical queries that
are $\alpha/4$-accurate,
Algorithm~\ref{alg} converges to a
$(\C,\alpha)$-multiaccurate hypothesis in
$T = O\left(\frac{\ell_\D(f_0;y)}{\alpha^2}\right)$ iterations.
\end{lemma}
We state this lemma in terms of a class
$\C$ but the proof reveals that any nontrivial update
that $\A$ returns suffices to make progress.
\begin{proof}
We begin by considering the effect of the multiplicative weights update as
a univariate update rule.
Suppose we use the multiplicative weights update rule
to compute $f_{t+1}(x)$ to be proportional to
$f_t(x) \cdot e^{-\eta c_t(x)}$ for some $c_t(x)$.
We can track how $\ell_x(f;y)$ changes based on the choice
of $c_t(x)$.
\begin{multline}
\ell_x(f_t;y) - \ell_x(f_{t+1};y)\\ =
y(x) \cdot \log\left(\frac{f_{t+1}(x)}{f_t(x)}\right)
+ (1-y(x))\cdot \log\left(\frac{1-f_{t+1}(x)}{1-f_t(x)}\right)\label{eqn:difference}
\end{multline}
Recall $f_t(x) = \frac{q_t(x)}{1+q_t(x)}$, so
$1-f_t(x) = \frac{1}{1+q_t(x)}$.  Thus, we can rewrite
(\ref{eqn:difference}) as follows.
\begin{align}
&y(x) \cdot \log\left(\frac{q_{t+1}(x)}{q_t(x)}\right)
+ (1-y(x))\cdot \log\left(\frac{1}{1}\right)
- \log\left(\frac{1+q_{t+1}(x)}{1+q_t(x)}\right)\\
&= -\eta c_t(x) y(x) + 0
- \log\left(\frac{1+q_{t+1}(x)}{1+q_t(x)}\right)\label{eqn:mw}
\end{align}
where (\ref{eqn:mw}) follows by the multiplicative weights
update rule implies $q_{t+1}(x) = e^{-\eta c_t(x)}q_t(x)$
for $x \in S_t$.
Next, we expand the final logarithmic term.
\begin{align}
- \log\left(\frac{1+q_{t+1}(x)}{1+q_t(x)}\right)
&= - \log\left(\frac{1+q_{t}(x)e^{-\eta c_t(x)}}{1+q_t(x)}\right)\\
&\ge - \log\left(\frac{1+q_{t}(x)(1-\eta c_t(x)+\eta^2 c_t(x)^2)}{1+q_t(x)}\right)\label{eqn:quadratic}\\
&\ge - \log\left(1 - \frac{q_{t}(x)}{1+q_t(x)}(\eta c_t(x)-\eta^2 c_t(x)^2)\right)\\
&\ge \eta c_t(x) f_t(x) - \eta^2c_t(x)^2\label{eqn:one}
\end{align}
where (\ref{eqn:quadratic}) follows by upper bounding the
Taylor series approximation for $e^z$ for $z\ge -1$; and
(\ref{eqn:one}) follows by the fact that $f_t(x) \in [0,1]$.
Combining the expressions, we can simplify as follows.
\begin{align}
(\ref{eqn:mw}) &\ge
-\eta c_t(x) y(x)
+ \eta c_t(x) f_t(x) - \eta^2 c_t(x)^2\\
&=\eta c_t(x) \cdot (f_t(x) - y(x)) - \eta^2c_t(x)^2
\label{eqn:linear}
\end{align}
Thus, we can express the change in $\ell_x(f_t;y)
- \ell_x(f_{t+1};y)$ after an
update based on $c_t(x)$ in terms of the inner product
between $c_t$ and $f-y$.  In this sense, we can express the local
progress during the update at time $t$ in terms of some
global progress in the objective.

When we update $x \in \X$ simultaneously according to $c$, we can express the
change in expected cross-entropy as follows.
\begin{align}
&\ell_\D(f_t;y) - \ell_\D(f_{t+1};y)\\
&\ge \eta \cdot\E_{x \sim \X}[c_t(x) \cdot (f_t(x) - y(x))] - \eta^2 \cdot \E_{x \sim \X}[c_t(x)^2]\\
&\ge \eta \langle c_t, f_t - y \rangle
- \eta^2\\
&\ge \eta (\alpha/2 - \eta)
\label{ineq:approximation}
\end{align}
where (\ref{ineq:approximation}) follows from
the fact that we assumed that our estimates
of the statistical queries were $\alpha/4$-accurate and that we update based on
$c_t$ if $\langle c_t,f-y \rangle$ is at least
$3\alpha/4$ according to our estimates.
Thus, taking $\eta = \alpha/4$, then we see
the change in expected cross-entropy over
$\X$ is at least $\alpha^2/16$, which shows the
lemma.
\end{proof}

\section{Linear convergence from gradient learning}

Here we show that given an auditing algorithm $\A$ that learns the cross-entropy gradients accurately, Algorithm~\ref{alg} converges linearly.  Consider the following auditor $\A_\ell$.
We assume the norms and inner products are estimated accurately using $D \sim \D^m$.

\begin{figure}[ht!]
{\refstepcounter{algorithm} \label{auditor}{\bf Algorithm~\thealgorithm:}} $\A_{\ell}$ -- smooth cross-entropy auditor 

\fbox{\parbox{\columnwidth}{
{\bf Given:}
\begin{itemize}
\item hypothesis $f:\X \to [0,1]$;
\item class of functions $\C \subseteq [-B,B]^\X$; accuracy parameter $\alpha > 0$;
\item smoothing parameter $L$;
\item validation data $D \sim \D^m$;
\end{itemize}

{\bf Let:}
\begin{itemize}
\item $\eps \gets \frac{\langle \grad_f\ell, f-y\rangle^2}{\norm{\grad \ell}^2\norm{f-y}^2}$
\hfill\texttt{// approx factor based on angle between grad and f-y}
\item $\H \gets \set{h \in \C : \norm{h}^2 \le L\cdot \ell(f;y)}$
\hfill\texttt{// audit over l2-bounded version of C}
\item $h_f \gets \textrm{argmin}_{h \in \H} \norm{h - \grad_f\ell(f;y)}^2$
\end{itemize}
if $\ell(f;y) \le \alpha$ or $\norm{h_f - \grad_f\ell(f;y)}^2 > \frac{\eps}{2} \cdot \norm{\grad_f \ell(f;y)}^2$:\\
\phantom{For }{\bf return} $h(x) = 0$
\hfill\texttt{// cross-entropy small or hf bad approx to deriv}\\
else:\\
\phantom{For }{\bf return} $h_f$
}}
\end{figure}

We claim that this auditor learns the partial derivative function in a way that guarantees linear convergence.
\begin{proposition*}[Restatement of Proposition\ref{prop:learning}]
Let $\alpha,B, L > 0$ and $\C \subseteq [-B,B]^\X$.
Suppose we run Algorithm~\ref{alg} on initial model
$f_0$ with auditor
$\A_\ell$ defined in Algorithm~\ref{auditor}.
Then, Algorithm~\ref{alg} converges in
$T = O\left(L \cdot \log(\ell_\D(f_0;y)/\alpha)\right)$
iterations.
\end{proposition*}
\begin{proof}
Note that when $\A_\ell$ returns $h(x) = 0$, then Algorithm~\ref{alg}
terminates.  Thus, we will bound the
number of iterations until $\ell_\D(f;y)$ at most than $\alpha$.
For notational convenience, we denote
$\grad_f \ell_\D(f;y)$ as
$\grad_f \ell$.

By the definition of $\eps$ and the termination condition, we know that
if $\A_\ell$ returns $h_f(x) \neq 0$
then $h_f$ satisfies the following inequality.
\begin{align}
\norm{h_f - \grad_f \ell}^2 &\le \frac{1}{2}\cdot \frac{\langle \grad_f \ell, f-y \rangle^2}{\norm{f-y}^2}\\
&\le \frac{1}{2}\cdot \frac{\langle \grad_f \ell, f-y \rangle^2}{\norm{f-y}^2} + \frac{1}{16} \norm{\grad_f\ell}^2\\
&= \norm{\frac{\langle \grad_f\ell,f-y \rangle}{\norm{f-y}^2}(f-y) - \frac{\grad_f\ell}{4}}^2
\end{align}
Using this inequality, we can bound the inner product between $h_f$ and $f-y$.
\begin{align}
&\langle h_f, f-y \rangle\\ &=
\langle \grad_f \ell, f-y \rangle
+ \langle h_f - \grad_f \ell, f-y \rangle\\
&\ge \langle \grad_f \ell, f-y \rangle -
\norm{\frac{\langle \grad_f\ell,f-y \rangle}{\norm{f-y}^2}(f-y) -\frac{ \grad_f\ell}{4}} \cdot
\norm{f-y}\\
&\ge \langle \grad_f \ell, f-y \rangle
- \langle \grad_f \ell, f-y \rangle \cdot \frac{\norm{f-y}^2}{\norm{f-y}^2}
+ \frac{1}{4}\cdot \langle \grad_f\ell,f-y \rangle
\label{ineq:positive}\\
&\ge \frac{1}{4} \cdot \ell_\D(f;y)\label{ineq:linear:convex}
\end{align}
where (\ref{ineq:positive}) follows from the fact
that $\grad_f \ell$ and $f-y$ are positively
correlated; and (\ref{ineq:linear:convex}) follows
by convexity of $\ell_\D$.

Thus, using the analysis of the multiplicative weights update from Section~\ref{app:mw}, we
can see that the progress in cross-entropy can be bounded as
\begin{align}
\ell_\D(f_{t};y) - \ell_\D(f_{t+1};y) &\ge 
\frac{\eta}{4} \cdot \ell_\D(f_t;y) - \eta^2
\cdot \norm{h_{f_t}(x)}^2\\
&\ge (\frac{\eta}{4} - \eta^2 L )\cdot \ell_\D(f_t;y)\label{ineq:smoothness}
\end{align}
where (\ref{ineq:smoothness}) follows from the fact that $h_f$ is drawn
from a class with Euclidean norm bounded as $\norm{h_f}^2 \le L \cdot \ell_\D(f;y)$.

Rearranging and taking $\eta = \frac{1}{8L}$, we arrive at the following inequality that implies linear convergence.
\begin{align}
\ell_\D(f_{t+1};y) &\le (1-\frac{\eta}{4} + \eta^2 L)
\ell_\D(f_t;y)\\
&\le e^{-1/64L}\ell_\D(f_t;y)
\end{align}
Thus, after $O\left(L\cdot \log(\ell_\D(f_0;y)/\alpha)\right)$,
then the cross-entropy will drop below
$\alpha$.
\end{proof}

\end{document}